\documentclass{article} 

\usepackage{microtype}
\usepackage{graphicx}
\usepackage{subfigure}
\usepackage{booktabs} 

\usepackage{hyperref}

\usepackage[accepted]{icml2020}

\usepackage{times}
\usepackage{soul}
\usepackage{url}
\usepackage{caption}
\usepackage{amsmath}
\usepackage{amsthm}
\usepackage{amssymb}
\usepackage{amsfonts}
\usepackage{booktabs}
\usepackage{algorithm}
\usepackage{algorithmic}
\newtheorem{theorem}{Theorem}

\usepackage{helvet}  
\usepackage{courier}  
\usepackage{subfigure}
\usepackage{amsfonts}       
\usepackage{nicefrac}       
\usepackage{amsmath}
\usepackage{stfloats}

\newcommand{\beqn}{\begin{eqnarray}}
\newcommand{\eeqn}{\end{eqnarray}}
\newcommand{\beqnx}{\begin{eqnarray*}}
	\newcommand{\eeqnx}{\end{eqnarray*}}

\def\W{{\bf W}}

\def\b{{\bf b}}

\newcommand{\thmref}[1]{Theorem~\ref{#1}}

\icmltitlerunning{Deep Graph random Process for Relational-Thinking-based  Speech Recognition}

\begin{document}

\twocolumn[
\icmltitle{Deep Graph Random Process for\\ Relational-Thinking-Based Speech Recognition}




\begin{icmlauthorlist}
	\icmlauthor{Hengguan Huang}{1}
	\icmlauthor{Fuzhao Xue}{1}
	\icmlauthor{Hao Wang}{2}
	\icmlauthor{Ye Wang}{1}
\end{icmlauthorlist}

\icmlaffiliation{1}{National University of Singapore}
\icmlaffiliation{2}{Massachusetts Institute of Technology}

\icmlcorrespondingauthor{Ye Wang}{wangye@comp.nus.edu.sg}

\icmlkeywords{Bayesian deep learning, Bayesian nonparametric deep learning, Graph random process, Relational thinking }

\vskip 0.3in
]



\printAffiliationsAndNotice{} 

\begin{abstract}

Lying at the core of human intelligence, relational thinking is characterized by initially relying on innumerable unconscious percepts pertaining to relations between new sensory signals and prior knowledge, consequently becoming a recognizable concept or object through coupling and transformation of these percepts. Such mental processes are difficult to model in real-world problems such as in conversational automatic speech recognition (ASR), as the percepts (if they are modelled as graphs indicating relationships among utterances) are supposed to be innumerable and not directly observable. In this paper, we present a Bayesian nonparametric deep learning method called deep graph random process (DGP) that can generate an infinite number of probabilistic graphs representing percepts.  We further provide a closed-form solution for coupling and transformation of these percept graphs for acoustic modeling.  Our approach is able to successfully infer relations among utterances without using any relational data during training. Experimental evaluations on ASR tasks including CHiME-2 and CHiME-5 demonstrate the effectiveness and benefits of our method.

\end{abstract}

\section{Introduction}

Relational thinking is believed to be a fundamental human learning process. In this type of thinking, people obtain sensory signals such as sounds, sights, and smells without consciously perceiving them, but those signals nonetheless lead to thought outcomes. For example, take the process of a baby learning to listen or speak \cite{alexander2016relational}. This learning process is characterized by initially relying on innumerable unconscious percepts pertaining to relations between current sensory signals and prior knowledge \cite{malmberg2012relationship,murai2018optimal}, then subsequently discovering recognizable concepts or objects through coupling and transformation of these percepts \cite{alexander2016relational}. In contrast, relational reasoning is a higher-level learning process that intentionally or consciously reasons about relations among objects or concepts. While relational reasoning has inspired perspectives of artificial intelligence \cite{hudson2019learning,smolensky1987connectionist}, relational thinking is largely unexplored in solving machining learning problems.

Human conversation is essentially a process of exchanging thoughts between two or more talkers.  Speech processing (or, specifically, speech recognition) is one of the critical components of this highly complex process, and neurobiology research acknowledges that this component is connected to human thinking \cite{hickok2007cortical,birjandi2015review}. In contrast, automatic speech recognition (ASR) has long been treated as a typical pattern matching problem in the machine learning area \cite{rabiner1989tutorial,hinton2012deep,amodei2016deep}. Ironically, this task is generally decomposed into two parts, namely acoustic modeling and linguistic decoding, and the essential relational thinking process is ignored \cite{alexander2016relational}. This process is difficult to incorporate into a traditional speech recognition system because 
the percepts (e.g. mental impressions formed while hearing sounds) involved in relational thinking are supposed to be innumerable and not directly observable. However, the dialogue history during the conversation might reflect such underlying mental processes, allowing an indirect way of modeling speech processing \cite{derix2014speech}. 
Even when only using one previous dialogue embedding as an additional input to an acoustic model, the proposed framework in \cite{moses2019real} achieves significantly better results on a spoken Q\&A dataset. More powerful representation learning might be achieved by weighting over multiple context inputs (e.g. utterances)  \cite{palm2018recurrent,chen2016end,pundak2018deep,zoph2016multi,kim2018dialog}.  

Despite recent advances, it is still challenging to explicitly modeling the percept of relational thinking for acoustic modeling due to percept's unconscious nature. Currently, the hybrid acoustic recurrent neural network hidden Markov model (RNN-HMM) still outperforms end-to-end encoder-decoder approaches for acoustic modeling in many aspects~\cite{luscher2019rwth}, even though the former is getting more popular. Generally, RNNs do a good job of capturing long-term temporal dependencies of sequential inputs but a poor job at representing complex relationships.  Therefore, when handling tasks that require relational thinking, the first-order dependencies between adjacent hidden states imposed by RNNs may hinder the extraction of complex structural information. One option to get around this problem is to process such data, which has a highly complex structure, with some members from the family of graph neural networks. However, most existing techniques either require the input to be graph data \cite{kipf2016variational, kipf2016semi,bojchevski2018netgan}, or to be supervised toward training targets of graph structure \cite{samanta2018designing}. As such, they are not applicable for a task such as ours that marries relational thinking with acoustic modeling, in that the relations involved in percepts are difficult to obtain.     

To alleviate these deficiencies in acoustic modeling, we propose a new Bayesian nonparametric deep learning method called deep graph random process (DGP) that can model percepts involved in relational thinking as probabilistic graphs without using any relational data during training. Specifically, a percept in our work is modeled as relations between a current utterance and its history. We assume the probability of the existence of such a relation to be close to zero due to the unconsciousness of the percept. Given an utterance and its history,  we generate an infinite number of percepts represented by probabilistic graphs contained in the DGP, in which each node depicts a representation of an utterance and each edge corresponds to relation between two nodes.
We assume that the edge of percept graph is distributed according to a Bernoulli distribution with the probability of edge existence being close to zero. 

It is computationally intractable to combine innumerable graphs by simply summing over their adjacency matrix. We therefore find an analytical solution by creating an equivalent new graph where the edge is represented by a Binomial variable.  Since Binomial distribution of such variable involves an infinity as one of its parameters, we further find a close form solution for inference and sampling of such Binomial distribution via an approximate Gaussian distribution with bounded approximation errors.
To transform the new graph to be ``conscious'' or task-specific, we weight each edge of the new graph using another Gaussian variable conditioned on the edge drawn from the Binomial variable. Subsequently, we calculate the graph embedding \cite{kipf2018neural} over the transformed graph and using it as an additional input for acoustic modeling. To jointly optimize above components, we 
adopt variational inference framework and successfully derive an effective evidence lower bound (ELBO).  

The experiments on CHiME-2 \cite{chime2} and CHiME-5 \cite{barker2018fifth} show that our new model consistently outperforms baseline models for the speech recognition task. Notably, we demonstrate the effectiveness of our method in learning relations among utterances via both qualitative and quantitative studies on  synthetic relational SWitchBoard \cite{godfrey1992switchboard} data.      
\section{Related Work}
\label{bg}
\subsection{Bayesian Deep Learning of Graphs} 

Several Bayesian deep learning models for graphs have recently been proposed. For example, relational stacked denoising auto-encoders (RSDAE)~\cite{RSDAE} were developed as a principled model to incorporate graph structures into probabilistic auto-encoders, significantly improving representation learning. As a follow-up work, relational deep learning (RDL)~\cite{RDL} is a supervised and fully Bayesian version of RSDAE to directly tackle the task of link prediction. Along a different line of research, graph auto-encoders (GAEs) \cite{kipf2016variational} were proposed to learn real-world graph data in an unsupervised training manner. Different from RSDAE and RDL, they employ a graph convolutional networks (GCN) \cite{kipf2016semi} encoder to represent nodes using low-dimensional vectors, and use a decoder to reconstruct the adjacency matrix. These models have found applications in discovering chemical molecules \cite{samanta2018designing}, modeling citation networks \cite{kipf2016variational}, and constructing knowledge graphs \cite{chen2018variational}. 

Regularized Graph Variational Autoencoders (RGVAE) \cite{pan2018adversarially} improve upon GAEs through regularizing the output distribution of the decoder with an adversarial regularization framework. 
\cite{bojchevski2018netgan} further extends this approach with a random-walk-based generator. However, these approaches assume the model data to be a static graph,  limiting their model generalizability in handling real-world problems with dynamic graphs. Variational graph RNN (VGRNN) \cite{hajiramezanali2019variational} attempts to mitigate this problem by combining GCN, RNN, and GAEs, allowing the evolution of dynamic graphs to be captured.  Though these existing works are successful in generating graphs, static or dynamic, they require graph annotations. Therefore, they are not applicable to our task in which the annotations of relations among utterances are not available during training.    
\subsection{Variational Acoustic Modeling}
An RNN-HMM acoustic model is a major component of a hybrid RNN-HMM ASR system \cite{lstm:graves.timit.icassp2013}. It can be viewed as an ``HMM states classifier''. Specifically,  given $M$ training utterances $\{U_{i}\}_{i=1}^{M}$, where $U_{i}$  involves  a sequence of acoustic features  $ \mathbf{X}_{i}= \{\mathbf{x}_{i,1},\mathbf{x}_{i,2}, . . ., \mathbf{x}_{i,T}\}$ and training labels $ \mathbf{Y}_{i}= \{\mathbf{y}_{i,1},\mathbf{y}_{i,2}, . . ., \mathbf{y}_{i,T}\}$. 
An RNN taking as input the acoustic feature $\mathbf{x}_{i,t}$  is adopted to estimate posterior probabilities for $K$ HMM states of context-dependent phones:
\begin{equation}
    \hat{\mathbf{y}}_{i,t}=\mathrm{softmax}(\mathrm{RNN}(\mathbf{x}_{i,t}))
\end{equation} 
where $\hat{\mathbf{y}}_{i,t}$ is the HMM state prediction. Such a model can be optimized by minimizing the negative log-likelihood or cross-entropy: $ \sum_{i}^{M} - \log P(\mathbf{Y}_{i}|\mathbf{X}_{i})$. 

Many uncertainties can be encountered when modeling speech signals using a RNN-HMM acoustic model. For instance, 
the background noise has a complicated influence on the speech signal. The RNN-HMM acoustic model is limited in handling such uncertainties because an RNN is essentially a deterministic function. A variational RNN (VRNN) \cite{chung2015recurrent} has been developed to cope with this  limitation. It introduces a latent variable $\mathbf{z}_{i,t}$  to capture the uncertainty of acoustic features at time $t$. Such a latent variable is assumed to have a Gaussian prior distribution $p(\mathbf{z}_{i,t}|\mathbf{h}_{i,t-1})$ dependent on the previous RNN hidden state $\mathbf{h}_{i,t-1}$. Its posterior distribution is approximated by a variational distribution $q(\mathbf{z}_{i,t}|\mathbf{x}_{i,t},\mathbf{h}_{i,t-1})$, allowing  the use of the evidence lower bound (ELBO) for joint learning and inference. It can be written as:
\begin{equation}
    \begin{split}
        \sum_{i=1}^{M} \{ 
&\mathrm{KL}(q(\mathbf{Z}_{i} |\mathbf{X}_{i},\mathbf{H}_{i})
||
p(\mathbf{Z}_{i} | \mathbf{H}_{i})) 
\\&- \mathbb{E}_{\mathbf{Z}_{i}}[\log P(\mathbf{Y}_{i}|\mathbf{X}_{i},\mathbf{Z}_{i})]
\}
    \end{split}
\end{equation}
where latent variables $\mathbf{Z}_{i}=\{\mathbf{z}_{i,1},\mathbf{z}_{i,2}, . . ., \mathbf{z}_{i,T}\}$ and hidden states   $\mathbf{H}_{i}=\{\mathbf{h}_{i,0},\mathbf{h}_{i,1}, . . ., \mathbf{h}_{i,T-1}\}$. The sampling from the posterior distribution is achieved by using re-parameterization based Monte Carlo (MC) estimation \cite{kingma2013auto}. This model can then be trained via stochastic gradient descent. It has been observed that the gradient obtained using this technique is more stable than that of the score function estimator~\cite{glynn1990likelihood}. 

One major limitation of the VRNN-HMM acoustic model is that the learned latent distribution exhibits non-interpretable representations because the approximating distributions are assumed to take a general form which are lacking in expressive power. For instance, frame-level latent variables adopted by VRNN are not powerful enough to describe the relational structure among utterances. Addressing this issue is one of our model's focuses.  

\section{Relational Thinking Modelling }

\subsection{Problem Formulation and Preliminary}  

\label{mdl}
 \begin{figure} 
 \vskip -0.05in
    \begin{center}
        \resizebox{0.45\textwidth}{!}{\includegraphics{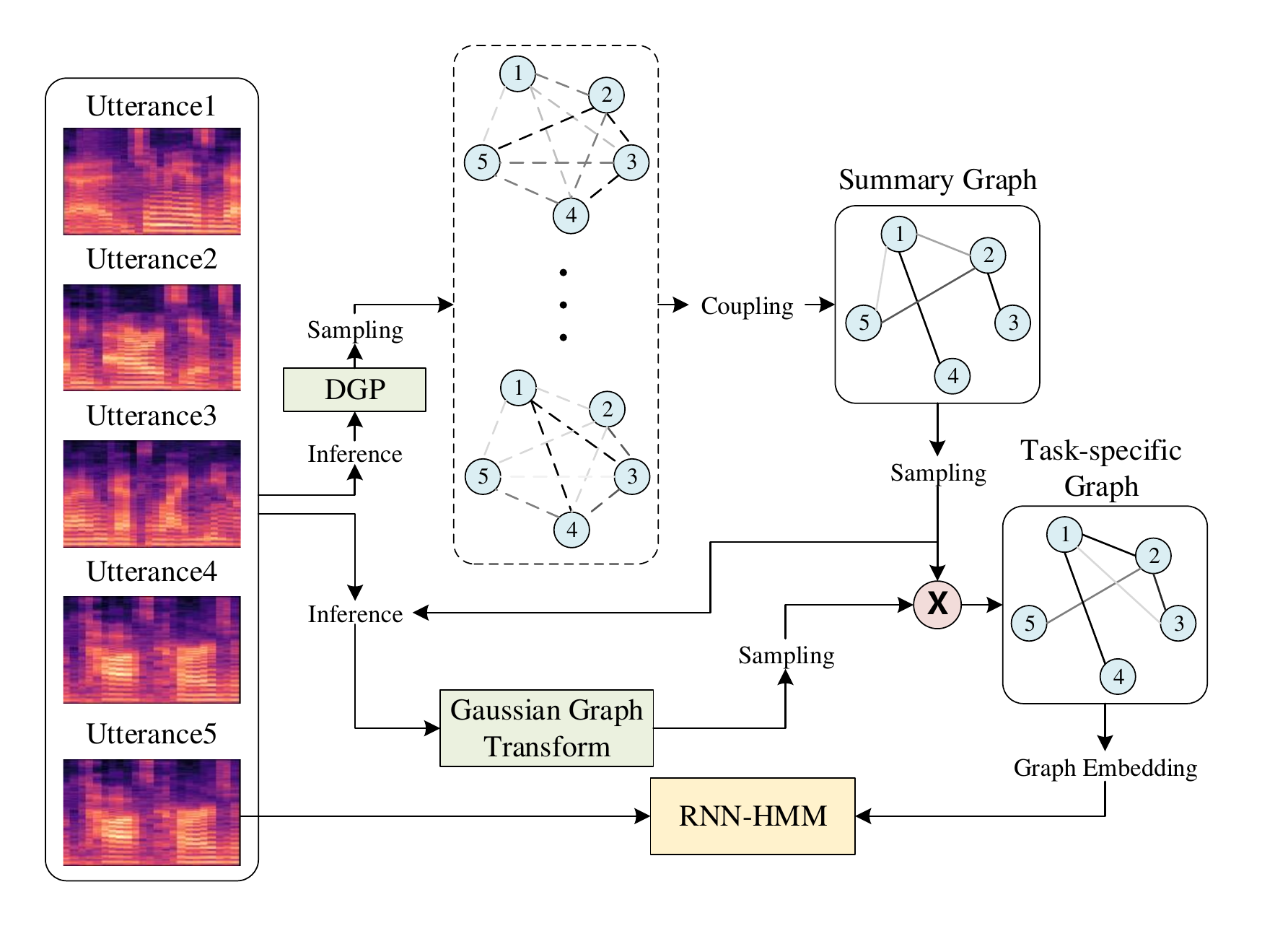}}
        \caption{ Architecture of RTN for acoustic modeling  }
        \label{fig:overview}
    \end{center}
    \vskip -0.2in
\end{figure}

In a natural conversation, there is a relational structure among consecutive utterances which depends on the speaker's intention in producing those utterances. For instance, such a structure may contains some discourse relations, such as question-answer-pair, contrast, comment and so on.
One common way of representing such a structure is a graph in which a node corresponds to an utterance embedding and an edge corresponds to the relationship between the two connected nodes, e.g. the co-occurrence of dialogue acts \cite{Stolcke-etal:2000} of two utterances.  According to relational thinking \cite{alexander2016relational,malmberg2012relationship,murai2018optimal}, before producing such complex relational data among utterances, the listener is assumed to unconsciously generate an infinite amount of similar data except that the probability of edge existence is very small. These data are not recognizable until they are somehow combined and transformed for a specific task, e.g. speech recognition.

Suppose we are given an utterance $U_{i}$  involving a sequence of acoustic features  $ \mathbf{X}_{i}= \{\mathbf{x}_{i,1},\mathbf{x}_{i,2}, . . ., \mathbf{x}_{i,T}\}$ and training labels $ \mathbf{Y}_{i}= \{\mathbf{y}_{i,1},\mathbf{y}_{i,2}, . . ., \mathbf{y}_{i,T}\}$ as well as $o$ previous utterances $U_{i-o:i-1}$. We aim to simulate relational thinking by initially constructing an infinite number of graphs  $\{G^{(k)}\}_{k=1}^{+\infty}$, where $G^{(k)}(V^{(k)},E^{(k)})$ is a graph representing the $k$-th percept, with $V^{(k)}$ and $E^{(k)}$ as the node and edge sets. Consequently, these percept graphs are combined and further transformed via a graph transform $\mathbf{S}$. In this paper, a node $v^{(k)}_{i}$ corresponds to the embedding of utterance $U_{i}$ for the $k$-th percept , and an element $\alpha^{(k)}_{i,j}$ of the adjacency matrix $\mathbf{A}^{(k)}$  indicates an edge between node $i$ and node $j$. It is worth noting that we assume that the probability of edge existence in such graphs is close to zero. Our ultimate goal is to model the distribution of training labels conditioned on acoustic features and these graphs as well as a graph transform, i.e. $P(\mathbf{Y}_{i}|\mathbf{X}_{i},\{G^{(k)}\}_{k=1}^{+\infty}, \mathbf{S})$, with a closed-form solution.

\subsection{Deep Graph Random Process}

Deep graph random process (DGP) is a stochastic process designed to describe the latent mechanisms governing the generation of an infinite number of  probabilistic graphs representing percepts. A DGP contains a fixed number of nodes that are shared among percept graphs. Such nodes may represent sensory input from multiple sensory modalities such as olfaction, vision and audition, depending on the specific problem of interest. For example, in acoustic modeling, each node represents an utterance whose embedding is calculated by a neural network encoding a sequence of acoustic features for the $i$-th utterance:
\begin{equation}
\mathbf{v}_{i}=f_{\theta}(\mathbf{X}_{i})
\end{equation} 
where $f_{\theta}$ is a neural network. 

At DGP's core are a series of deep Bernoulli processes (DBP) as  building blocks, each being responsible to generate edges between two nodes of DGP. DBP is a special type of Bernoulli process we propose in which the probability of the existence of the edge is assumed to be close to zero due to the unconsciousness of the percept. 
Specifically, by considering an infinite number of edges $\{\alpha^{(k)}_{i,j}\}_{k=1}^{+\infty}$ sampled from the DBP between node $i$ and node $j$, we have:
\begin{equation}
\{\alpha^{(k)}_{i,j}\}_{k=1}^{+\infty}\sim \mathrm{DBP}(\mbox{Bern}(\lambda_{i,j}))
\end{equation}    
where $\mbox{Bern}(\lambda_{i,j})$ is the Bernoulli distribution with the probability of the edge existence $\lambda_{i,j}$. As a result, our DGP is capable of generating innumerable probabilistic graphs:
\begin{equation} \label{6}
 \{G^{(k)}\}_{k=1}^{+\infty} \sim \mathrm{DGP}(\mathbf{v}_{i-o:i},\{\mathrm{DBP}(\mbox{Bern}(\lambda_{*,*}))\})
\end{equation}
where
$\{\mathrm{DBP}(\mbox{Bern}(\lambda_{*,*}))\}$ refers to a collection of DBP involved in DGP.

\subsubsection{Coupling of innumerable probabilistic graphs}

Coupling is one of the key steps of relational thinking. The goal of coupling is to obtain a summary of an infinite number of percepts. However, it is computationally intractable to combine innumerable percept graphs by simply summing over $\mathbf{A}^{(k)}$ for each graph generated by DGP. We seek to construct an equivalent graph that can serve as a summary or representation of the original innumerable graphs.
We refer to such graph as \textit{summary graph}. They keep the original node set and update each edge by sampling from a Binomial distribution:
\begin{equation}
\tilde{\alpha}_{i,j}=\sum_{k=1}^{+\infty}\alpha_{i,j}^{(k)},~~~~~~~~ \tilde{\alpha}_{i,j}\sim \mathcal{B}(n,\lambda_{i,j}),
\end{equation} 
where $n\rightarrow +\infty$ and $\lambda_{i,j}\rightarrow 0$.


\subsubsection{Inference and sampling of edges of summary graph}

We assume the edge $(i,j)$ of summary graph is distributed according to a Binomial distribution $\mathcal{B}(n,\lambda_{i,j})$, with $n\rightarrow +\infty$ and $\lambda_{i,j}\rightarrow 0$.  Drawing inspiration from VRNN \cite{chung2015recurrent}, the parameters of the approximate posteriors are estimated from an RNN encoding acoustic features. However, unlike VRNN, DGP contains the approximate posterior $q(\tilde{\alpha}_{i,j} |\mathbf{X}_{i-o:i})$, whose inference and sampling cannot be directly solved in a computationally tractable manner due to the infinity $n$.     

\begin{theorem}\label{thm:t1}
Let $\mathcal{N}(\mu,\sigma^{2})$ denotes a Gaussian distribution with $\mu<1/2$, and let $\mathcal{B}(n,\lambda)$  denotes a Binomial distribution with $n\rightarrow +\infty$ and $\lambda\rightarrow 0$, where  $n$ is increasing while $\lambda$ is decreasing.    
There exists a real constant $m$ such that if $m=n\lambda$ and if we define:
\begin{align*}
 f_{1}(x) &= \mathrm{KL}(\mathcal{N}(x,x(1-x)) || \mathcal{N}(\mu,\sigma^{2}) )\\
 f_{2}(x) &= \mathrm{KL}(\mathcal{N}(x,x(1-x)) || \mathcal{N}(n \lambda,n \lambda(1-\lambda) )\\
 f_{2}^{*} &= \min_x f_{2}(x), \  where \  x \in (0,1)
\end{align*}
we have that:  $ f_{1}(x)$ attains its minimum on the interval $(0,1)$ and $ f_{2}(x)-f_{2}^{*}$ is bounded on the interval $(0,\sqrt{2}/2-1/2)$, with:
\begin{equation}
x=m=\frac{1+l-\sqrt{1+l^{2}}}{2}, \ where  \ l=\frac{2\sigma^{2}}{1-2\mu}
\nonumber
\end{equation}
\end{theorem}

Suppose we are given a Gaussian distribution $\mathcal{N}(\tilde{\mu}_{i,j},\tilde{\sigma}_{i,j}^{2})$, whose parameter $\tilde{\mu}_{i,j}$ is specifically parameterized by the neural network that can guarantee that $\tilde{\mu}_{i,j}<1/2$. By De Moivre–Laplace theorem \cite{sheynin1977laplace}, we have that $\mathcal{N}(n \lambda_{i,j},n \lambda_{i,j}(1-\lambda_{i,j})$ is a good approximation  for $\mathcal{B}(n,\lambda_{i,j})$. They are asymptotically equivalent as $n$ increases. Let $m_{i,j}=n\lambda_{i,j}$, with \thmref{thm:t1} (see Supplement for the Proof of Theorem 1), direct parameterization of both the infinite parameter $n$ and the near-zero parameter $\lambda_{i,j}$  can be avoided, which in turn allows for the re-parametrization trick of~\cite{kingma2013auto} to be used. This trick draws samples from such Binomial distribution via its Gaussian proxy $\mathcal{N}(m_{i,j},m_{i,j}(1-m_{i,j}))$.

\subsection{Application of DGP for Acoustic Modelling}
\label{rpu}

Another essential aspect of relational thinking is that it transforms innumerable unconscious percepts into a recognisable notion of knowledge. 
Here, we aim to extract an informative representation from the summary graph representing innumerable percept graphs for our downstream task: acoustic modelling. This is achieved by transforming the summary graph through weighting each edge with a Gaussian variable $s_{i,j}$:
\begin{equation}
    \bar{\alpha}_{i,j}=s_{i,j}*\tilde{\alpha}_{i,j}
\end{equation}
We further assume such Gaussian variable to be conditioned on the edge $\tilde{\alpha}_{i,j}$ of the summary graph, in avoiding the distribution of such Gaussian variable (if it is independent of the edge $\tilde{\alpha}_{i,j}$) behaves randomly when some sample values of the edge $\tilde{\alpha}_{i,j}$ are close to zero. It is defined as:
\begin{equation}
s_{i,j}|\tilde{\alpha}_{i,j}\sim \mathcal{N}(\tilde{\alpha}_{i,j}* \mu_{i,j},\tilde{\alpha}_{i,j}*\sigma^{2}_{i,j})
\end{equation}
We refer to such operations as a \textit{Gaussian graph transform}. The resultant graph is called a \textit{task-specific graph}.

We then follow \cite{kipf2018neural} to extract the graph embedding $\mathbf{e}_{i}$ from the transformed graph with node $\mathbf{v}_{i}$ corresponding to the current utterance. It can be written as:
\begin{equation}
\mathbf{e}_{i}= \sum_{
(j,k)\in \left \{(j,k) | j<k\leq i, (j,k) \in \bar{E} \right \} 
 }\bar{\alpha}_{j,k} \bar{f_{\theta}}([\mathbf{v}_{j},\mathbf{v}_{k}])
\end{equation}
where $\bar{f_{\theta}}$ is a neural network and $\bar{E}$ is the edge set of the transformed graph.

Next, we use the generated graph embedding as an additional input of our acoustic model. We refer to the whole framework as a relational thinking network (RTN) (Figure \ref{fig:overview}).  In this paper, we adopt the simple recurrent unit \cite{sru} as the basic building block of the RTN. The SRU simplifies the architecture of LSTM and dramatically increases computational speed with nearly no ASR performance drop. 
The updating formulas of our RTN are: 
\begin{eqnarray}
&\left[\hat{\mathbf{r}}_{i,t},\hat{\mathbf{f}}_{i,t},\hat{\mathbf{c}}_{i,t}\right] = \mathbf{W}_x \left [ \mathbf{x}_{i,t},  \mathbf{e}_{i}\right ] +\mathbf{b} \\
&\mathbf{r}_{i,t}  =  \sigma (\hat{\mathbf{r}}_{i,t}) \\
&\mathbf{f}_{i,t}  =  \sigma (\hat{\mathbf{f}}_{i,t}) \\
&\mathbf{c}_{i,t}  =  \mathbf{f}_{i,t} \odot \mathbf{c}_{i,t-1}+ (1-\mathbf{f}_{i,t}) \odot \hat{\mathbf{c}}_{i,t} \\
&\mathbf{h}_{i,t}  =  \mathbf{r}_{i,t} \odot \mathbf{c}_{i,t} + (1-\mathbf{r}_{i,t}) \odot W_h  \left [ \mathbf{x}_{i,t},   \mathbf{e}_{i}\right ] 
\end{eqnarray}
where 
$\mathbf{r}_{i,t}$ 
is the reset gate output,
$\mathbf{f}_{i,t}$
is the forget gate output,
$\mathbf{c}_{i,t}$
is the memory cell output,
$\W_x$ and $\W_{h}$ are the weight matrices,
$\b$ is the gate bias vector, 
$\mathbf{h}_{i,t}$
is the hidden state output,
any quantity with a `hat' (e.g. $\hat{\mathbf{c}}_{i,t}$) is the activation value of the quantity before an
activation function is applied,
$\odot$ is the element-wise multiplication operation, and 
$\sigma$ is the sigmoid function.

\subsection{Learning}
\label{lr}

We adopt variational inference to jointly optimise DGP, the Gaussian graph transform, and the acoustic model. DGP can be equivalently represented as two type of random variables: Bernoulli variables related to edges of the percept graph and Binomial variables related to edges of the summary graph. Although these two random variables take different forms in terms of probability distributions, we can use these different random variables to describe the same random process data. Therefore, specifying the Binomial variables of a DGP completely determines the graph random process as a whole. The resulting objective is to maximize the evidence lower bound (ELBO):
\begin{equation}
\begin{split}
\sum_{i=1}^{M} \{ 
&\mathrm{KL}(q(\tilde{\mathbf{A}},\mathbf{S} |\mathbf{X}_{i-o:i})
||
p(\tilde{\mathbf{A}},\mathbf{S} |\mathbf{X}_{i-o:i})) 
\\&- \mathbb{E}_{\tilde{\mathbf{A}},\mathbf{S}}[\log P(\mathbf{Y}_{i}|\mathbf{X}_{i},\tilde{\mathbf{A}},\mathbf{S})]
\}
\end{split}
\end{equation}
where the adjacency matrix of the summary graph $\tilde{\mathbf{A}}=[\tilde{\alpha}_{i,j}]$; the Gaussian graph transform matrix $\mathbf{S}=[\tilde{s}_{i,j}]$. (see Section \ref{imp} for the parameterization of the approximate posterior $q(\tilde{\mathbf{A}},\mathbf{S} |\mathbf{X}_{i-o:i})$ and the prior $p(\tilde{\mathbf{A}},\mathbf{S} |\mathbf{X}_{i-o:i})$) 
Since each element of the Gaussian graph transform matrix is conditioned on the Binomial variable for the same edge of the summary graph, the $\mathrm{KL}$ term can be further written as:
\begin{equation}
\begin{split}
\sum_{(i,j)\in \tilde{E}} \{
&\mathrm{KL}(\mathcal{B}(n,\tilde{\lambda}_{i,j})
||
\mathcal{B}(n,\tilde{\lambda}_{i,j}^{(0)}) 
\\& + \mathbb{E }_{\tilde{\alpha}_{i,j}   }   [ 
 \mathrm{KL}(\mathcal{N}(\tilde{\alpha}_{i,j}  * \mu_{i,j},\tilde{\alpha}_{i,j}*\sigma^{2}_{i,j})
\\&|| \
\mathcal{N}(\tilde{\alpha}_{i,j}* \mu_{i,j}^{(0)},\tilde{\alpha}_{i,j}* {\sigma_{i,j}^{(0)}}^{2})] \
\}
\end{split}
\end{equation}
Unfortunately, while calculation of the second term is straightforward,  the first $\mathrm{KL}$ term is computationally intractable as $n\rightarrow +\infty$. 

\begin{theorem}\label{thm:t2}
Suppose we are given two Binomial distributions, $\mathcal{B}(n,\lambda)$ and $\mathcal{B}(n,\lambda^{0})$ with $n\rightarrow +\infty$, $\lambda^{0}\rightarrow 0$ and $\lambda\rightarrow 0$ , where  $n$ is increasing while $\lambda$ and $\lambda^{0}$ are decreasing.  There exists a real constant $m$ and another real constant $m^{(0)}$, such that  if $m=n\lambda$ and $m^{(0)}=n\lambda^{(0)}$ and if  $\lambda > \lambda^{(0)} $, we have:
\begin{equation}
\begin{split}
   \mathrm{KL}(\mathcal{B}(n,\lambda) ||& \mathcal{B}(n,\lambda^{0}))
 <m\log\frac{m}{{m}^{(0)}} 
\\&+ (1-m )  \log\frac{1-m+m^{2}/2}{1-m^{(0)}+{m^{(0)}}^{2}/2} \nonumber 
\end{split}
\end{equation}
\end{theorem}

By \thmref{thm:t2} (the proofs are provided in the supplementary material), we have a closed-form solution that is irrelevant to $n$ for the ELBO.

\subsection{Detailed Implementation of RTN}
\label{imp}
\paragraph{Node Embedding of DGP} In our following experiments, we adopt the neural network $f_{\theta}$ to calculate the node embedding of DGP.  The architecture of such a neural network has 6 SRU layers (each with 1024 hidden states), which is firstly followed by a max-pooling layer and then a single-layer multi-Layer perceptron (MLP). Here we show the detailed formulation:
\begin{align*}
    \mathbf{H}_{i}&= \mathrm{SRU}(\mathbf{X}_{i}) 
\\   \tilde{\mathbf{v}}_{i}&= \max \limits_{t}(\mathbf{H}_{i})
\\ \mathbf{v}_{i}&=   \mathrm{ReLU}(\mathrm{MLP}(\tilde{\mathbf{v}}_{i}))
\end{align*}
where SRU has 6 stacked layers; $ \max$ is an element-wise max operation over the frames in the utterance; the input and output size of the MLP is 1024 and 128 respectively. 

\paragraph{Inference of Binomial Edge Variables of DGP } For approximate posterior on the Binomial edge variable, before calculation of $m_{i,j}$, theorem \ref{thm:t1} requires a Gaussian distribution $\mathcal{N}(\tilde{\mu}_{i,j},\tilde{\sigma}_{i,j}^{2})$ with $\tilde{\mu}_{i,j}<1/2$. We adopted two three-layer MLP (with 128 hidden nodes per layer) taking as input node embeddings $\mathbf{v}_{i-9:i}$  and compute $\tilde{\mu}_{i,j}$ and $\tilde{\sigma}_{i,j}$ respectively. To avoid the explosion of $\frac{1}{1-2\tilde{\mu}_{i,j}}$, we introduce another variable $n_{i,j}$ and define it 
as:
\begin{equation}
  n_{i,j}=\frac{1}{1-2\tilde{\mu}_{i,j}} = \mathrm{softplus}(\tilde{\mu}_{i,j}) + \epsilon 
\end{equation}
such that $n_{i,j}$ is lower bounded by $\epsilon$. In our experiments, $\epsilon$ was set to 0.01 (such hyperparameter can be tuned to further improve
performance). Then $m_{i,j}$ can be calculated as:
\begin{equation}
    m_{i,j}=\frac{1+2n_{i,j}\tilde{\sigma}_{i,j}^{2}-\sqrt{1+ 4n_{i,j}^{2}\tilde{\sigma}_{i,j}^{4}   }}{2}
\end{equation}
The bound variable $n_{i,j}$ helps to avoid the explosion of $\log(m_{i,j})$ involved in our final objective. For the prior on the Binomial edge variable, the parameter $m_{i,j}^{(0)}$ is learned by a three-layer MLP (with 128 hidden nodes per layer) taking as input node embeddings $\mathbf{v}_{i-9:i}$.

\paragraph{Gaussian Graph Transform} To compute the parameters of the approximate posterior  $\mathcal{N}(\tilde{\alpha}_{i,j}* \mu_{i,j},\tilde{\alpha}_{i,j}*\sigma^{2}_{i,j})$ involved in Gaussian graph transform,  two three-layer MLP (with 128 hidden nodes per layer) taking as input node embeddings $\mathbf{v}_{i-9:i}$ was adopted to calculate $\mu_{i,j}$ and $\sigma_{i,j}$ respectively. The parameters of the corresponding prior is obtained in a similar way. 

\section{Experiments}
\label{exp}

We perform the preliminary experiments for the proposed RTN on a reading speech recognition corpus, CHiME-2 \cite{chime2}.
We then evaluate our method on CHiME-5 \cite{barker2018fifth},  a more challenging conversational speech recognition dataset where the data is collected from everyday home environments. We finally investigate the interpretability of our model on a synthetic relational speech dataset: synthetic relational SWitchBoard \cite{godfrey1992switchboard} (RelationalSWB). 

\begin{table}
  \caption{Model configuraions for all datasets and the training time for CHiME-2. L: number of layers; N: number of hidden states per layer;   P: number of model parameters; T: Training time per epoch (hr).}
\label{tbl-chime2_time}
\begin{footnotesize}
  \centering
  \begin{tabular}{lcccc}
    \toprule
{\sf Model}  & {\sf L} & {\sf N} & {\sf P} & {\sf T} \\ 
    \midrule
LSTM \cite{huang2019recurrent}                                             & 3           &2048             & 130M                                & 0.71                     \\ 
SRU  \cite{huang2019recurrent}                                           & 12            &2048           & 156M                                  & 0.32                     \\ 
RPPU \cite{huang2019recurrent}                                       & 12            &1024            & 142M                                & 0.37                     \\ 
\midrule
Our SRU \cite{lei2017simple}                                     & 12            &1280            & 63M                                & 0.09                     \\ 
VSRU \cite{chung2015recurrent}                                     & 9            &1024            & 66M                                & 0.09                     \\ 

RRN \cite{palm2018recurrent}                                       & 9            &1024            & 64M                                & 0.09                     \\ 
RTN (Ours)                                     & 9            &1024            & 70M                                & 0.11                     \\ 
    \bottomrule
  \end{tabular}
  \end{footnotesize}
\end{table}

\subsection{Datasets}

\subsubsection{CHiME-2}
CHiME-2 corpus is designed for noise-robust speech recognition tasks. It was generated by convolving clean Wall Street Journal (WSJ0) \cite{garofalo2007csr} utterances with binaural room impulse responses (BRIRs) and real background noises at signal-to-noise ratios (SNRs) in the range [-6,9] dB. The training set contains 7138 simulated noisy utterances. The transcriptions are based on those of the WSJ0 training set. The development and test sets contain 2460 and 1980 simulated noisy utterances respectively. The WSJ0 text corpus is used to train a trigram language model with a vocabulary size of 5k.
 \begin{table}
	\caption{WER (\%) on test set of CHiME-2. }
	\label{tbl-chime2}
	\centering
	\begin{tabular}{ll}
		\toprule
		
		{\sf Model}  & WER  \\ 
		\midrule
		Kaldi DNN \cite{kaldi} & 29.1 \\
		LSTM      \cite{huang2019recurrent} & 26.1                  \\ 
		SRU       \cite{huang2019recurrent} & 26.2                  \\ 
		RPPU      \cite{huang2019recurrent} & 24.4              \\ 
		\midrule
		Our SRU \cite{lei2017simple}  & 25.8                  \\ 
		VSRU \cite{chung2015recurrent}  & 25.8                  \\ 
		RRN \cite{palm2018recurrent}  & 24.8                  \\ 
		RTN (Ours)  & \textbf{23.9}                  \\ 
		\bottomrule
	\end{tabular}
	\vskip -0.1in
\end{table}

 \subsubsection{CHiME-5}

CHiME-5 is the first large-scale corpus of real multi-speaker conversational speech in everyday home environments. It was originally designed for the CHiME 2018 challenge \cite{barker2018fifth}. Note that only the audio data recorded by binaural microphones is employed for training and evaluation in this experiment. 
The training dataset, development dataset and test dataset includes about 40 hours, 4 hours, and 5 hours of real conversational speech respectively. The evaluation was performed with a trigram language model trained from the transcription of CHiME-5.

\subsubsection{RelationalSWB}

RelationalSWB is a manually generated speech dataset based on the SWitchBoard (SWB) \cite{godfrey1992switchboard} conversational speech corpus, for which graph annotations among utterances are derived from the SWitchBoard dialogue act (SwDA) corpus \cite{jurafsky1997switchboard}. SWB training set includes about 260K utterances. SwDA extends it with dialogue act tags which are utterance-wise labels indicating the function of utterance in the dialog.

The training set of RelationalSWB contains 30K SWB training utterances. It is obtained by running the official script on Kaldi S5b \cite{povey2011kaldi}. We then select 1155 conversations that appear in both SwDA and SWB (not including Relational SWB training utterances) to construct the test set of RelationalSWB. This results in about 110K utterances.
Since SwDA only provides the utterance-wise dialogue act tags, we manually generated another dataset that contains binary relation labels between utterances. In doing so, we first generated the dialogue act tag pairs when the difference between two utterance indices is not more than 10. Then we ranked them by their frequencies. We used the top 20\% pairs as positive pairs and the remaining as negative pairs. Note that the segmentation scheme of utterances in SWB differs from that of SwDA. Therefore, we first detected the most similar utterance in SwDA for each utterance of SWB per dialogue using difflib\footnote{https://docs.python.org/3/library/difflib.html}; after that, the ground-truth relations of utterances on 
the test set could be obtained. 

\subsection{Feature Extraction and Preprocessing} 
The speech data in both CHiME-2 and CHiME-5 is pre-processed as 40-dimensional Mel-filterbank coefficients \cite{biem2001application} (for all neural-network-based models ), while it is 36-dimensional Mel-filterbank coefficients for RelationalWSB. All acoustic features are  calculated every 10ms. 
Input of all neural networks consists of the current frame together with its 4 future contextual frames. We performed speaker-level mean and variance normalization for the input to all models.

\subsection{Training Procedure}
All GMM-HMMs for CHiME-2 are trained using the standard Kaldi s5 recipe \cite{kaldi}. Note that the training recipe of CHiME-5 is modified from that of CHiME-2 as we only employed single-channel audio data for GMM-HMM training. They were then used to derive the state targets for subsequent RNN training through forced alignment for CHiME-2 and CHiME-5. Specifically, the state targets of CHiME-2 and CHiME-5 were obtained by aligning the training data with the DNN acoustic model through the iterative procedure outlined in \cite{iterative_procedure}. All RNNs were trained by optimizing the categorical cross-entropy using BPTT and SGD. We applied a dropout rate of 0.1 to the connections between recurrent layers.


\paragraph{Models}
We adopted SRU as the building block to construct all RNNs. For example, our VRNN implemented with SRU is called VSRU. We compare our proposed model with the following baseline models:
(i) SRU with 12 stacked layers; (ii) VSRU with 9 stacked layers in the decoder and 6 stacked layers in the encoder; (iii) RRN \cite{palm2018recurrent} with 9 stacked layers.
 \begin{figure}
   \vskip -0.3cm
    \begin{center}
        \resizebox{0.47\textwidth}{!}{\includegraphics{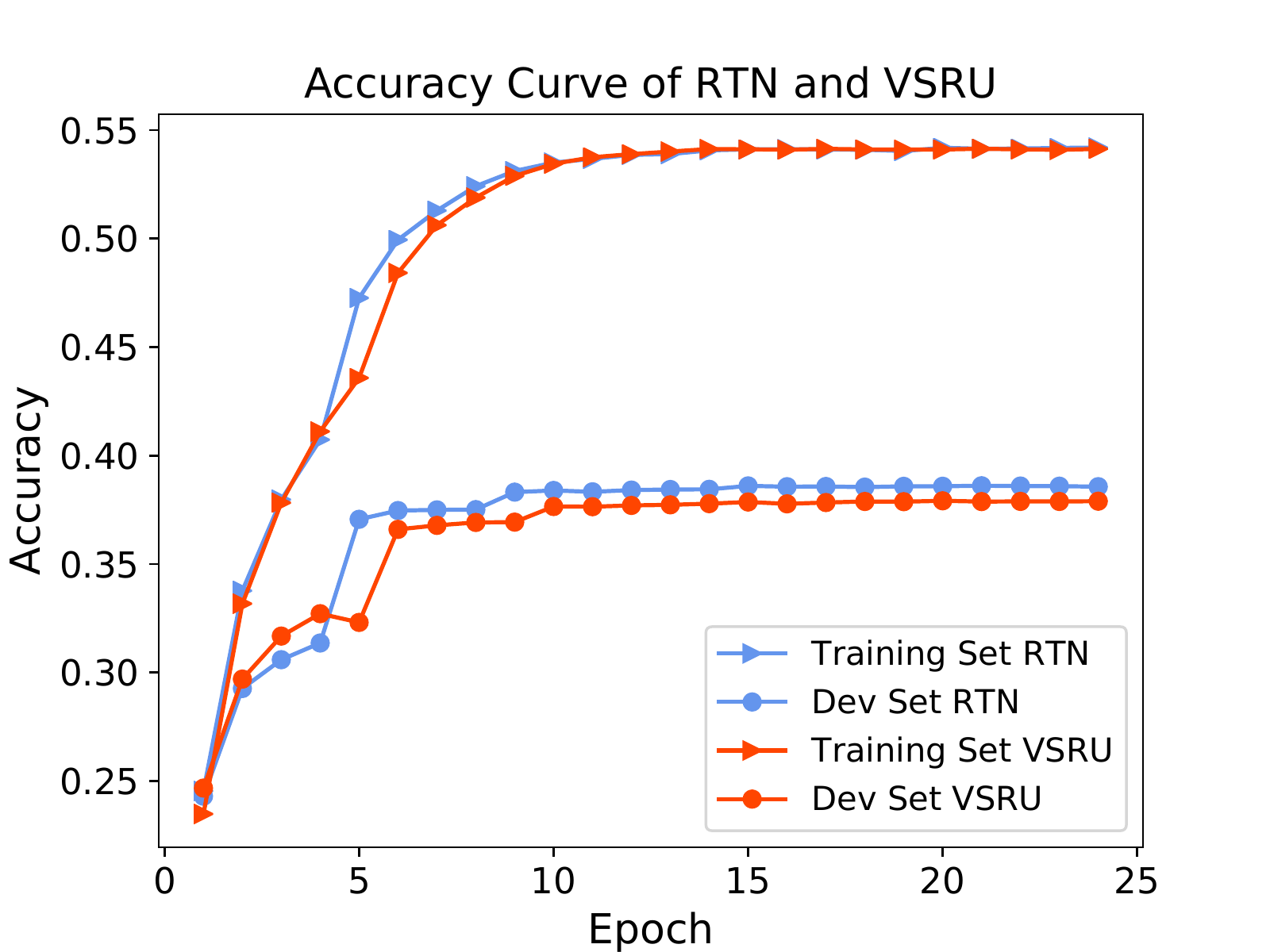}}
        \caption{ Frame Accuracy on CHiME-2 over multiple iterations  }
        \label{fig:acc}
    \end{center}
      \vskip -0.20in
\end{figure}
Among the baselines, RRN uses the same context information, i.e., multiple utterances, as our RTN. Other baselines use only one single utterance due to their model limitations. 
To ensure similar numbers of model parameters for different models, we set the number of hidden states per layer to 1280 for SRU and 1024 for VSRU, RRN, and RTN. Training with the original ELBO for VSRU and our RTN produces unstable results. Therefore, we follow the training strategy of $\beta$-VAE \cite{higgins2017beta} to reweight the importance of $\mathrm{KL}$ terms. The size of the latent vector of VSRU is set as 4 for CHiME-5 and 16 for CHiME-2 and RelationalSWB. 
Theoretically, our RTN can handle a very large pool of historical utterances. Considering the computational overhead, we set the number of historical utterances $o$ as 9, meaning that RTN can sequentially generate a relational structure for 10 utterances at a time, though it can be tuned to further improve performance. The size of node embedding $\mathbf{v}_{i}$ and graph embedding $\mathbf{e}_{i}$ are set as 128. 

\subsection{Results and Analysis}
\subsubsection{Preliminary study on CHiME-2}

Table~\ref{tbl-chime2_time} shows the configurations of baseline models and the new RTN model for all datasets. The training time per epoch for CHiME-2 is also reported.
In our experiments, the timing experiments used the PyTorch package and were performed on a machine running the Ubuntu operating system with a single Intel Xeon Silver 4214 CPU and a GTX 2080Ti GPU. Each model took around 25 iterations, and their average running time is reported. We can see that our SRU runs much faster than all models reported in \cite{huang2019recurrent} including SRU, due to the hardware optimization of SRU being adopted \cite{lei2017simple}. Besides, our RTN runs almost as fast as baseline models while having a similar number of parameters.  

\begin{figure}
    \begin{center}
        \resizebox{0.44\textwidth}{!}{\includegraphics{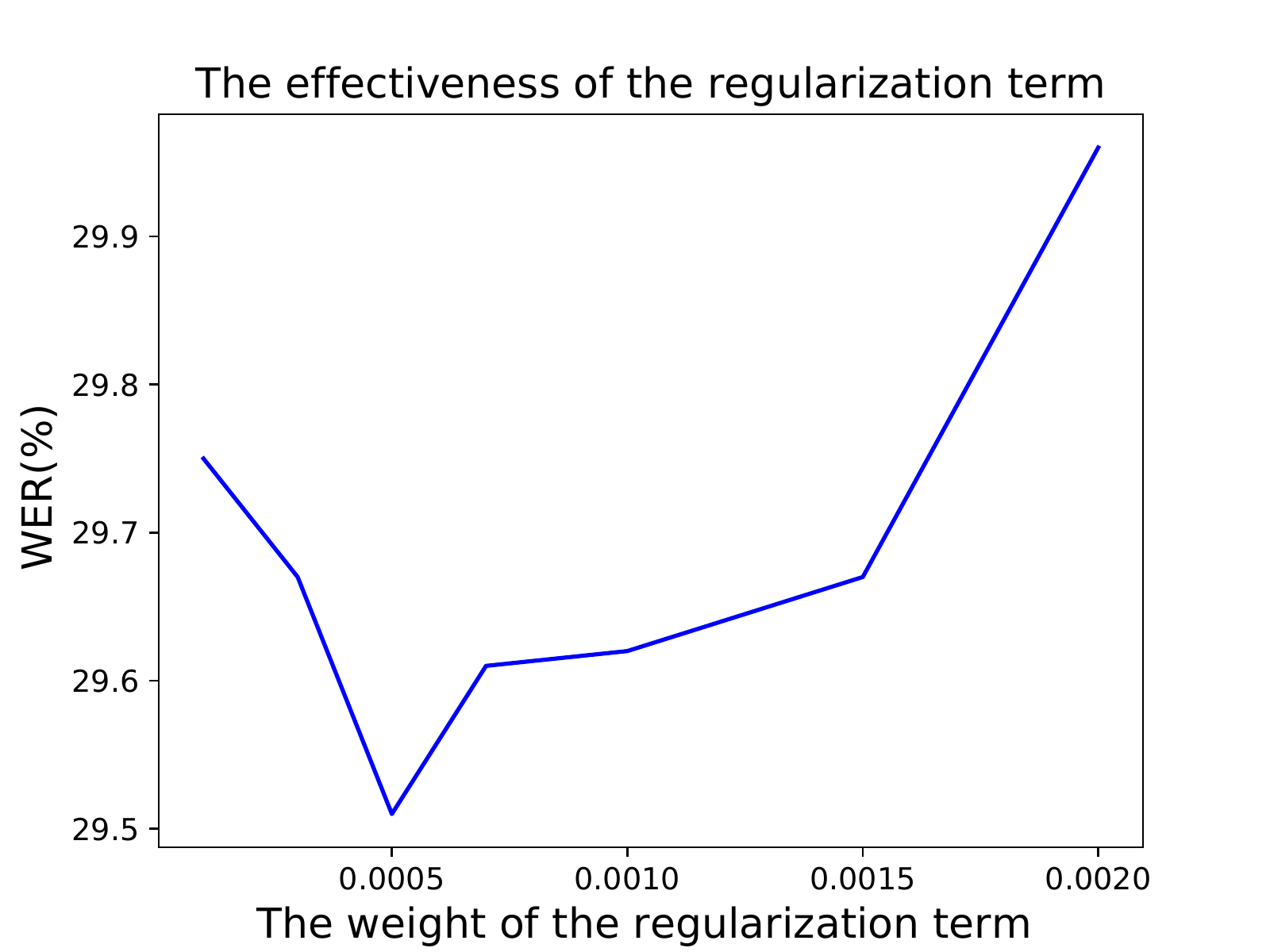}}
        \caption{WER on the development set of CHiME-2 by varying the weight of the regularization term. } 
        \label{fig:regu}
    \end{center}
      \vskip -0.20in
       
\end{figure}

Table~\ref{tbl-chime2} shows the word recognition performance of the baseline models and the new RTN model for CHiME-2.  First, we can see that SRUs, LSTM and VSRU achieve similar WERs. These baselines perform much better than the DNN baseline from Kaldi s5. Our RTN performs the best among all models in terms of WER, outperforming the RNNs by about 2.0\%. Compare with the state-of-the-art acoustic RRN and RPPU models, our RTN achieves 0.9\% and 0.5\% absolute WER reduction respectively. It is worth noting that our experimental configuration for all models is different from that of \cite{wang2016joint}, e.g. their system requires the clean data from WSJ0 to train an extra speech separation and to estimate training targets. We also report the detailed WERs as a function of the SNR in CHiME-2 in our supplemental materials.

To validate the effectiveness of the $\mathrm{KL}$ regularization term in RTN on CHiME-2, we varied its weight to find the best configuration (Figure \ref{fig:regu}). We obtained the best performance in the development set when the weight is about 0.0005. We therefore set it to 0.0005 as our final configuration based on this observation. These results demonstrate the effectiveness of our proposed objective function.

To evaluate the performance of the acoustic model by itself (i.e., without taking into account the language model's performance), we report the frame accuracy of VSRU and our RTN on CHiME-2.  It is displayed in Figure \ref{fig:acc}. Though two models achieve almost the same accuracy on training data set, note that RTN outperforms VSRU on Dev data set, yielding 0.8\% absolute frame accuracy improvement. This suggests that our RTN model can generalize significantly better than VSRU (see more details on the significance test in the Supplement). 

\subsubsection{Evaluation on Large-scale Real Conversational ASR}
\begin{table} 
	\caption{WER (\%) on eval of CHiME-5. }
	\label{tbl-chime5}
	\centering
	\begin{tabular}{lll}
		\toprule
		
		{\sf Model}   &WER  \\ 
		\midrule
		Kaldi DNN \cite{kaldi}             & 64.5            \\ 
		SRU \cite{lei2017simple}                & 62.6            \\ 
		VSRU \cite{chung2015recurrent}         & 61.6            \\ 
		RTN (Ours)               & \textbf{57.4}           \\ 
		\bottomrule
	\end{tabular}

\end{table}

We then conducted experiments on the first large-scale real conversation speech recognition dataset, CHiME-5.
The recognition results are shown in Table \ref{tbl-chime5}. We can see that our best baseline VSRU achieves a WER of 62.6\%, while our RTN has the lowest WER of 57.4\%. Overall, the RTN achieves 5.2\% and 4.2\% absolute WER reductions over SRU and VSRU respectively.

\begin{table}
	\caption{ Error rate(\%) of relation prediction on test set of RelationalSWB. }
	\label{tbl-SWB-Q}
	\centering
	\begin{tabular}{lll}
		\toprule
		
		{\sf Graph Type}  &  Err  \\ 
		\midrule
		Random Graph                           &50.0 \\
		Summary Graph                           & 28.6 \\
		Task-specific Graph                        & 28.7              \\ 

		\bottomrule
	\end{tabular}
	\vskip -0.10in
\end{table}

\subsubsection{Quantitative evaluation of DGP on RelationalSWB }

To quantitatively study the latent variables involved in DGP and Gaussian graph transform, we conducted analysis using the RTN acoustic model trained on RelationalSWB training set. We generated both summary graphs and task-specific graphs on the test set and then evaluated how well the edges of graphs match the ground-truth relations on RelationalSWB. 


To perform binary classification of the edges of the two graphs, we ranked the edges by their sample values and classified the top 20\% edges as `positive'. We report the error rate of such binary classification in Table \ref{tbl-SWB-Q}. We can see that our summary graph achieved an error rate of 28.6\%, which dramatically outperforms the baseline random graph by 21.4\%. This demonstrates our DGP's ability of generating meaningful relations among utterances without using any relational data during training. It marginally outperforms our task-specific graph. This is reasonable because after being transformed to fit with our downstream ASR task, it might lose information. We further perform a case study on the two graphs and seek to better understand such phenomena. 
\begin{table}[t]
\vspace{-2mm}
	\caption{WER (\%) on eval2000. }
	\label{tbl-SWB-30k}
	\centering
	\begin{tabular}{ll}
		\toprule
		
		{\sf Model}  & WER  \\ 
		\midrule
		Kaldi DNN \cite{kaldi}                                  & 26.8 \\
		SRU \cite{lei2017simple}                                   & 22.8              \\ 
		VSRU \cite{chung2015recurrent}                             & 22.6              \\ 
		RTN (Ours)                              & \textbf{20.8}             \\ 
		\bottomrule
	\end{tabular}
\vspace{-3mm}
\end{table}

The ASR performance of the RelationalSWB RTN is also reported.  Table \ref{tbl-SWB-30k} gives the WER comparison of all neural network models on eval2000 \cite{stolcke2000sri}. We can observe that RNN baseline systems achieve better WER than Kaldi DNN. Our RTN achieves the best WER of 20.8\%, yielding 8.8\% relative performance gain over the SRU baseline system. 

\begin{figure*}
     \vskip -0.15in
    \begin{center}
        \resizebox{0.76\textwidth}{!}{\includegraphics{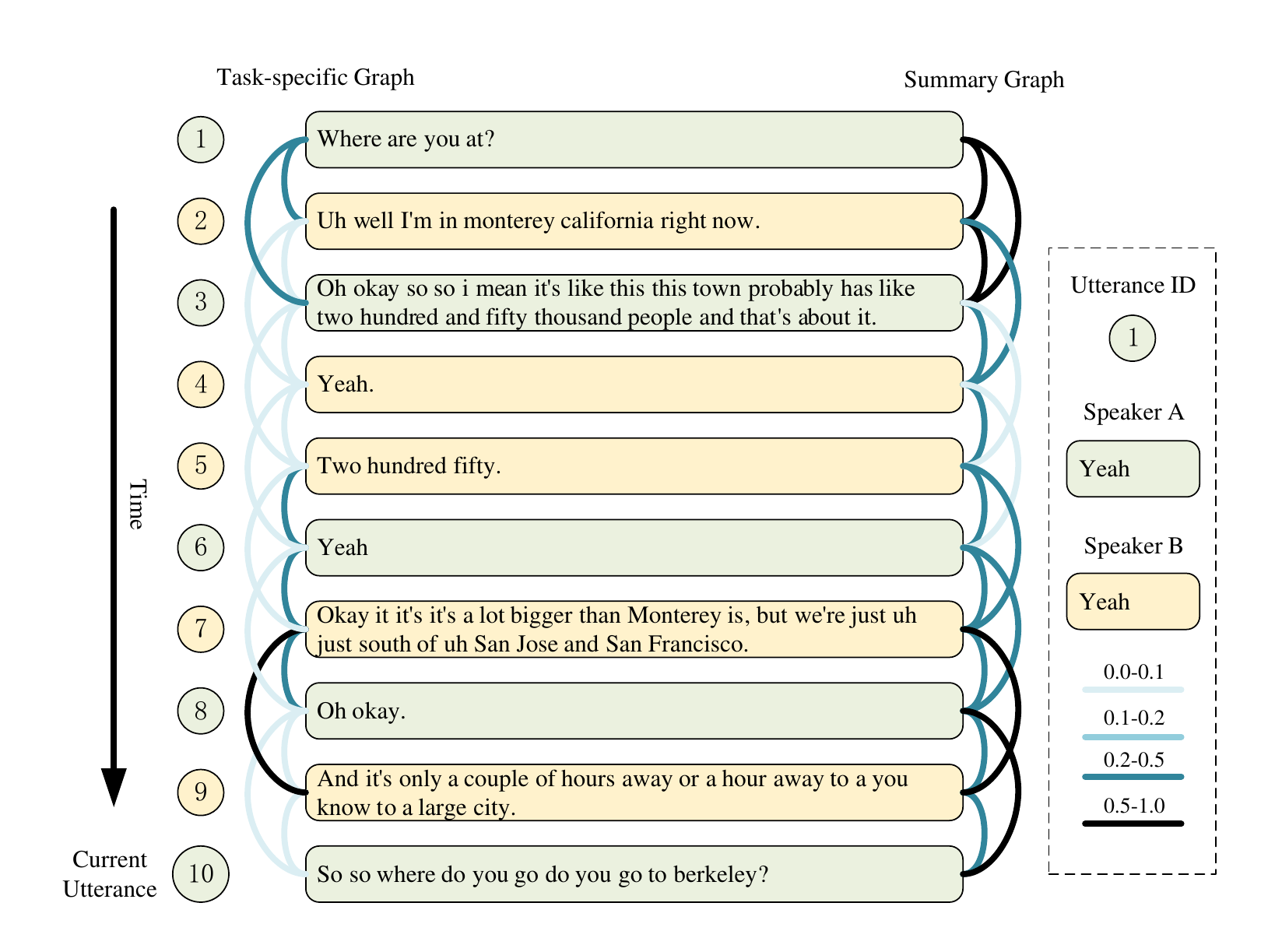}}
        \caption{Visualization of graphs generated by RTN }
        \label{fig:case}
    \end{center}
\end{figure*}

\subsubsection{Qualitative evaluation of DGP on SWB and SwDA }
We randomly selected ten sequential utterances from "sw02061-A\_008048-008171" to "sw02061-A\_010117-010354" in the RelationalSWB dataset. We then generated both the summary graph and task-specific graph (Figure \ref{fig:case}). For simplicity, the edge is removed when the difference between two node indices of the edge is more than 2.  The bottom utterance is the current one which the acoustic model is processing. We use min-max normalization to scale samples drawn from latent variables involved in the two graphs between 0 and 1. We color each edge according to such value. 

We can clearly see that the summary graph shown on the right is more densely connected than the left one. Indeed, less meaningful edges tend to be assigned with much smaller sample values, e.g. the edge between $10$-th utterance and $8$-th utterance. Interestingly, the first utterance labeled with ``Wh-question'' and the second one labeled with "Statement-non-opinion" exhibit ``strong'' relations; these ``strong'' relations are captured by both graphs. 
Again, this demonstrates that our model can generate complex relational structure among utterances.(see more examples of generated graphs in the Supplement)     

\section{Conclusion}
We propose a novel graph learning approach called deep graph random process (DGP) for relational thinking modelling. We show that our model can generate graphs representing complex relationships among utterances without using any relational data during training. We demonstrate that our DGP can be conveniently combined with a neural network model for a downstream task such as speech recognition via the graph Gaussian transform. Our experiments on CHiME-2 and CHiME-5 show that our method outperforms other RNN models in ASR. 

\section*{Acknowledgements}
The authors would like to thank the anonymous reviewers for their
insightful comments and suggestions, Dr. Vincent Y. F. Tan, Dr. David Grunberg and Dr. Graham
Percival for their assistance in proofreading the initial manuscript.
This project was partially funded by research grant R-252-000-A56-114
from the Ministry of Education, Singapore.

\nocite{langley00}
\bibliography{abbrev,mybib}
\bibliographystyle{icml2020}
\end{document}


\language0
\lefthyphenmin=2
\righthyphenmin=3

\maketitle

\section{Proof of Theorem 1}

\begin{theorem}\label{thm:t1}
Let $\mathcal{N}(\mu,\sigma^{2})$ denotes a Gaussian distribution with $\mu<1/2$, and let $\mathcal{B}(n,\lambda)$  denotes a Binomial distribution with $n\rightarrow +\infty$ and $\lambda\rightarrow 0$, where  $n$ is increasing while $\lambda$ is decreasing.    
There exists a real constant $m$ such that if $m=n\lambda$ and if we define:
\begin{align*}
 f_{1}(x) &= \mathrm{KL}(\mathcal{N}(x,x(1-x)) || \mathcal{N}(\mu,\sigma^{2}) )\\
 f_{2}(x) &= \mathrm{KL}(\mathcal{N}(x,x(1-x)) || \mathcal{N}(n \lambda,n \lambda(1-\lambda) )\\
 f_{2}^{*} &= \min_x f_{2}(x), \  where \  x \in (0,1)
\end{align*}
we have that:  $ f_{1}(x)$ attains its minimum on the interval $(0,1)$ and $ f_{2}(x)-f_{2}^{*}$ is bounded on the interval $(0,\sqrt{2}/2-1/2)$, with:
\begin{equation}
x=m=\frac{1+l-\sqrt{1+l^{2}}}{2}, \ where  \ l=\frac{2\sigma^{2}}{1-2\mu}
\nonumber
\end{equation}
\end{theorem}

\begin{proof}
The derivative of the function $f_{1}(x)$ over $x$ can be written as:
\begin{equation}
    {f}'_{1}(x) =x^{2}-(1+\frac{2\sigma^{2} }{1-2\mu})x + \frac{\sigma^{2} }{1-2\mu}
    \nonumber
\end{equation}
We set it as 0 and solve for $x$, giving 
\begin{equation} \label{1}
x= \begin{cases}
  \frac{1+l-\sqrt{1+l^{2}}}{2} & \text{ if } \mu<1/2 \\ 
\frac{1+l+\sqrt{1+l^{2}}}{2}  & \text{ if } \mu>1/2 
\end{cases}
,\ where \ l=\frac{2\sigma^{2}}{1-2\mu}
\end{equation}

Let $x=n\lambda$, the function $f_{2}(x)$ can be written as:
\begin{equation}
f_{2}(n\lambda)= \sqrt{      
\frac {1-n\lambda}
{1-\lambda}
}
+
\frac{1-\lambda}{2(1-n\lambda)}
-
1/2
\nonumber
\end{equation}

Let $g(n \lambda)=\lim_{\lambda \rightarrow 0}   f_{2}(n\lambda)$, we have
\begin{equation}
g(n\lambda)= \sqrt{      
1-n\lambda}
+
\frac{1}{2(1-n\lambda)}
-
1/2 
\nonumber
\end{equation}

Let $z=\sqrt{1-n\lambda}$, we have:
\begin{equation}
g(z)= z +1/(2z^{2}) -1/2 
\nonumber
\end{equation}
The derivative of function $g(z)$ over $z$ can be written as:
\begin{equation}
    {g}'(z) = 1 -1/z^{3}
    \nonumber
\end{equation}
Given that $z\in (0,1)$ , we have ${g}'(z)<0$. Then $g(z)$ attains its minimum 1 when $z$ approaches 1. Equivalently,  $f_{2}(n\lambda)$ attains its minimum 1 when $n\lambda$ approaches 0.

Considering Eq.\eqref{1}, we find that $n\lambda$ is bounded on $(0,1/2)$ if $\mu<1/2$,   

We then calculate the difference between $f_{2}(n\lambda)$ and its minimum. It can be written as
\begin{align*}
\Delta f_{2}(n\lambda)&= \lim_{\lambda \rightarrow 0}[ f_{2}(x)-f_{2}^{*}] \\ &=      g(n\lambda)-1 \\ &=\sqrt{      
1-n\lambda}
+
\frac{1}{2(1-n\lambda)}
-
3/2
\end{align*}
Let $m=n\lambda$, the derivative of function $\Delta f_{2}(m)$ over $m$ can be written as:
\begin{equation}
    \Delta {f}'_{2}(m) =  \frac{1-{(1-m)}^{3/2}}{2{(1-m)}^{2}}>0
    \nonumber
\end{equation}

Then $\Delta f_{2}(m)$ is monotonically increasing over $(0,1/2)$.
Therefore $\Delta f_{2}(m)$ is bounded on $(0,\sqrt{2}/2-1/2)$
\end{proof}

\section{Proof of Theorem 2}
\begin{theorem}\label{thm:t2}
Suppose we are given two Binomial distributions, $\mathcal{B}(n,\lambda)$ and $\mathcal{B}(n,\lambda^{0})$ with $n\rightarrow +\infty$, $\lambda^{0}\rightarrow 0$ and $\lambda\rightarrow 0$ , where  $n$ is increasing while $\lambda$ and $\lambda^{0}$ are decreasing.  There exists a real constant $m$ and another real constant $m^{(0)}$, such that  if $m=n\lambda$ and $m^{(0)}=n\lambda^{(0)}$ and if  $\lambda > \lambda^{(0)} $, we have:
\begin{equation}
\mathrm{KL}(\mathcal{B}(n,\lambda) || \mathcal{B}(n,\lambda^{0}))<m\log\frac{m}{{m}^{(0)}} 
+ (1-m )  \log\frac{1-m+m^{2}/2}{1-m^{(0)}+{m^{(0)}}^{2}/2} \nonumber
\end{equation}
\end{theorem}
\begin{proof}

Let $m=n\lambda$ and $m^{(0)}=n\lambda^{(0)}$, we have 
\begin{equation} \label{3}
\begin{split}
   \mathrm{KL}(\mathcal{B}(n,\lambda) || \mathcal{B}(n,\lambda^{(0)})) &=n\lambda\log\frac{\lambda}{{\lambda}^{(0)}} 
+ n(1-\lambda)  \log\frac{1-\lambda}{1-\lambda^{(0)}} 
\\& =  n\lambda\log\frac{n\lambda}{{n\lambda}^{(0)}}
+ n(1-\lambda)  \log\frac{1-\lambda}{1-\lambda^{(0)}} \\& =  m\log\frac{m}{m^{(0)}} 
+ n(1-\lambda) \log\frac{1-\lambda}{1-\lambda^{(0)}}
\end{split}
\end{equation}

We then take the right part,
\begin{equation}
    g= n(1-\lambda) \log\frac{1-\lambda}{1-\lambda^{(0)}} = (1-\lambda) \log\frac{{(1-\lambda})^{n}}{(1-\lambda^{(0)})^{n}} 
    \nonumber
\end{equation}
By Taylor series' theorem with Lagrange remainder, $g$ can be written as:
\begin{equation}
        g= (1-\lambda) \log\frac{
1-n\lambda + \frac{n(n-1)}{2} \lambda^{2} +R_{j=2}(-\lambda)    }
{    1-n\lambda^{(0)} + \frac{n(n-1)}{2} {\lambda^{(0)}}^{2} +R_{j=2}(-\lambda^{(0)})      } 
\nonumber
\end{equation}
There exists a $\theta \in (0,1)$ such that,
\begin{equation}
    R_{j=2}(x) = \frac{x^{3}(n-2)(n-1)n(1+x\theta)^{n-3}}{6} 
    \nonumber
\end{equation} 
Given that $n \rightarrow +\infty$  and $x \in (-1,1)$ , we have $(R)'_{j=2}(x) > 0$.
Therefore, $R_{j=2}(x)$ is  monotonically increasing over $(-1,1)$. Since $\lambda>\lambda^{0}$, we have 
\begin{equation} \label{7}
R_{j=2}(-\lambda)<R_{j=2}(-\lambda^{(0)}) 
\end{equation}

We then seek to prove: 
\begin{equation}
 k=  \frac{
1-n\lambda + \frac{n(n-1)}{2} \lambda^{2}     }
{    1-n\lambda^{(0)} + \frac{n(n-1)}{2} {\lambda^{(0)}}^{2}       } <1  
\nonumber
\end{equation}

Let $f(x) = n(n-1)x^{2}/2 - nx + 1$, we have 
\begin{equation}
    k = \frac{f(\lambda)}{f(\lambda^{(0)})}
    \nonumber
\end{equation}
Here, $f(x)$ is an U-shaped  parabola with axis $x=1/(n-1)$. 
By theorem \ref{thm:t1}, we have $n\lambda<1/2$, then we have $\lambda^{0}<\lambda<1/(n-1)$, then $f(x)$ is monotonically increasing over the support of $\lambda^{0}$ and $\lambda$, namely 
\begin{equation} \label{10}
   f(\lambda)<f(\lambda^{0}) 
\end{equation}
With Eq.\eqref{7} and  Eq.\eqref{10}, $g$ can be written as:
\begin{equation}
\begin{split}
            g&< (1-\lambda) \log\frac{
1-n\lambda + \frac{n(n-1)}{2} \lambda^{2}    }
{    1-n\lambda^{(0)} + \frac{n(n-1)}{2} {\lambda^{(0)}}^{2}       } \\& = (1-\lambda) \log\frac{
1-m +    m^{2}/2 -n\lambda^{2}/2    }
{    1-m^{(0)} +  {m^{(0)}}^{2}/2 -n{\lambda^{(0)}}^{2}/2       } 
\end{split}
\nonumber
\end{equation}
Similarly, let $h(x)= 1-x +    x^{2}/2$. It is an U-shaped  parabola with axis $x=1$ such that  
\begin{equation}
    -n\lambda^{2}/2<-n{\lambda^{(0)}}^{2}/2 
    \nonumber
\end{equation}
\begin{equation}
     1-m +    m^{2}/2< 1-m^{(0)} +  {m^{(0)}}^{2}/2
     \nonumber
\end{equation}
Then we have 
\begin{equation} \label{14}
\begin{split}
    g&<(1-\lambda) \log\frac{
1-m +    m^{2}/2    }
{    1-m^{(0)} +  {m^{(0)}}^{2}/2        }
\\&< (1-n\lambda) \log\frac{
1-m +    m^{2}/2    }
{    1-m^{(0)} +  {m^{(0)}}^{2}/2        }
\\&= (1-m) \log\frac{
1-m +    m^{2}/2    }
{    1-m^{(0)} +  {m^{(0)}}^{2}/2        }
\end{split}{}
\end{equation}
Combining Eq.\eqref{3} and Eq.\eqref{14} concludes the proof. 
\end{proof}

\section{Test of Significance}

The statistical significance test tool sc\_stats from National Institute of Standards and Technology (NIST) is used to compare our RTN and the baseline VSRU on CHiME-2 HMM states classification task. The test results find a significant difference in performance between the RTN and the VSRU at the level of $p<0.001$.

\section{Table of detailed WER (\%) on the CHiME-2 test set}

We report the detailed WERs as a function of the SNR in CHiME-2 shown in Table \ref{tbl-detail}. For all SNRs, the RTN outperforms other Baseline RNNs including LSTM, SRU by a large margin. It outperforms the state-of-the-art models including VSRU, RRN and RPPU for most SNRs. This suggests that incorporating the relational thinking into speech recognition lends itself to the model’s robustness.
\begin{table*} [h]
  \caption{Detailed WER (\%) on the CHiME-2 test set.}
\label{tbl-detail}
  \centering
  \begin{tabular}{lllllll}

    \toprule

{\sf Model} & -6 dB & -3 dB & 0 dB & 3 dB & 6 dB & 9 dB \\
    \midrule
LSTM   \cite{huang2019recurrent}                      & 42.4  & 33.5  & 26.7 & 21.1 & 17.3 & 15.3 \\ 
SRU   \cite{huang2019recurrent}                       & 42.5  & 34.0  & 26.2 & 22.2 & 17.4 & 15.1 \\ 
RPPU     \cite{huang2019recurrent}                    & 39.9  & 31.1  & 24.9 & 20.3 & 16.0 & \textbf{13.2} \\ 
\midrule 

Our SRU \cite{lei2017simple}                          &42.1	&33	&26.1	&20.7	&16.8	&15.1 \\ 
VSRU      \cite{chung2015recurrent}                     &41.5	&32.8	&26.2	&20.9	&16.9	&16.1 \\ 
RRN \cite{palm2018recurrent}                          &40.2	&32.1	&25.9	&20.2	&16.2	&14.0 \\ 
RTN (Ours)                         & \textbf{39.0}  & \textbf{30.4}  & \textbf{25.4} & \textbf{19.4} & \textbf{15.5} & 13.8 \\ 
    \bottomrule
  \end{tabular}
  \vspace{0.3in} 
\end{table*}

\newpage

\section{More examples of graphs generated by RTN}

\begin{figure*}[h]
     \vskip -0.15in
    \begin{center}
        \resizebox{0.82\textwidth}{!}{\includegraphics{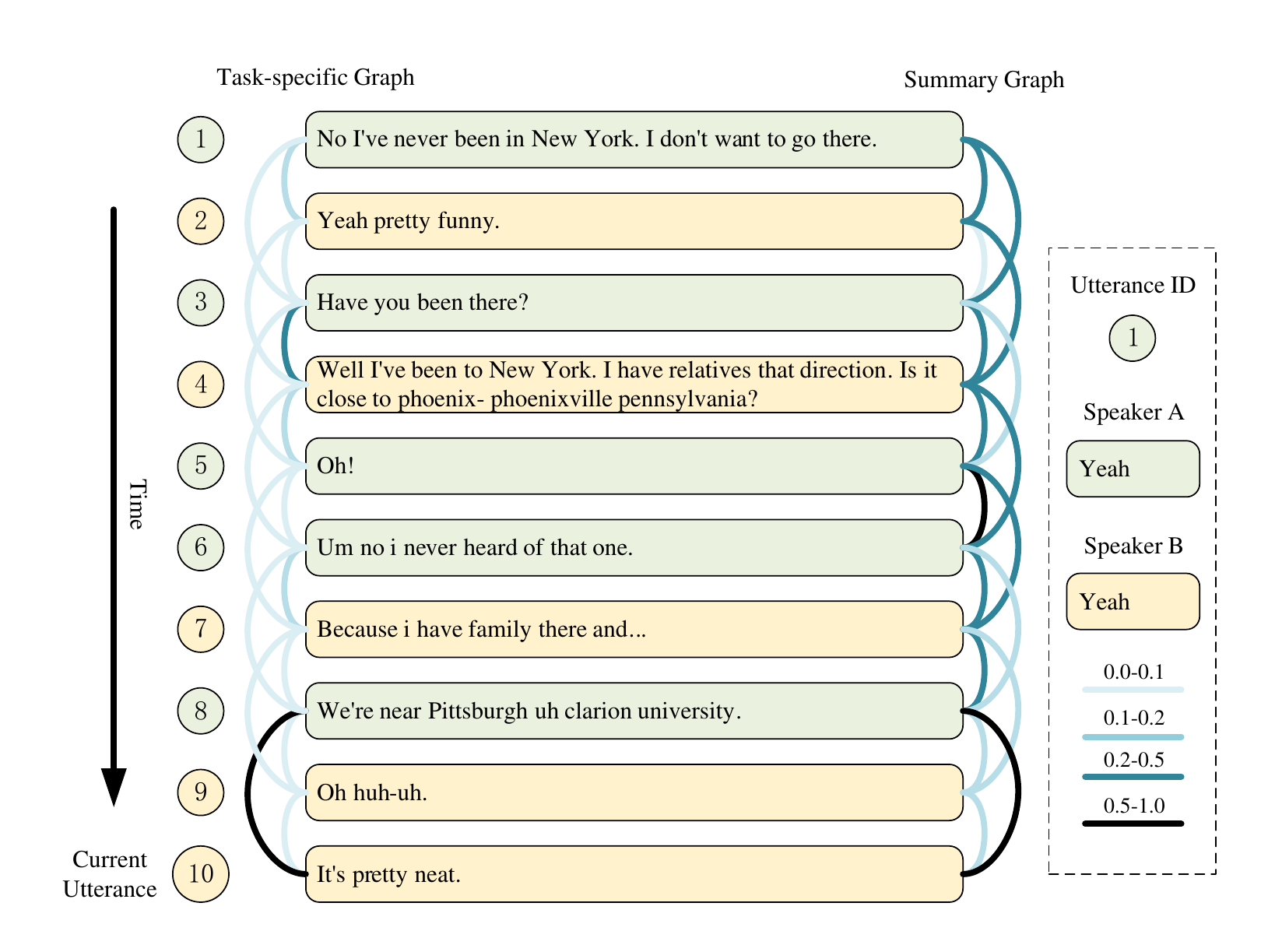}}
        \caption{Example of graphs generated by RTN: ten sequential utterances from "sw02262-A\_029098-029769" to "sw02262-B\_031645-031828"}
        \label{fig:case}
    \end{center}
\end{figure*}

\begin{figure*}[h]
     \vskip -0.15in
    \begin{center}
        \resizebox{0.82\textwidth}{!}{\includegraphics{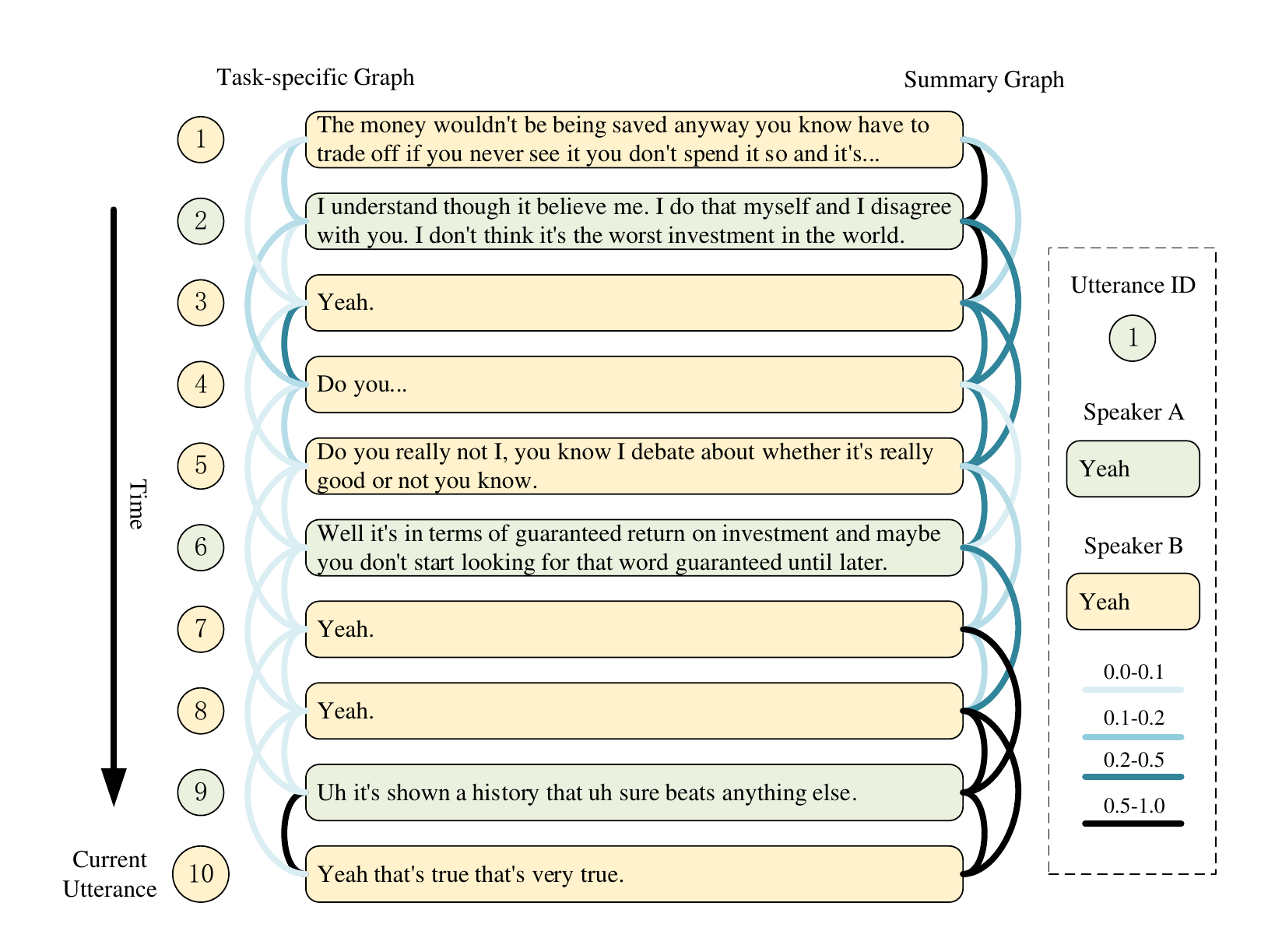}}
        \caption{Example of graphs generated by RTN: ten sequential utterances from "sw02062-B\_019277-020062" to "sw02062-B\_022871-023232" }
        \label{fig:case}
    \end{center}
\end{figure*}

\begin{figure*}[h]
     \vskip -0.15in
    \begin{center}
        \resizebox{0.82\textwidth}{!}{\includegraphics{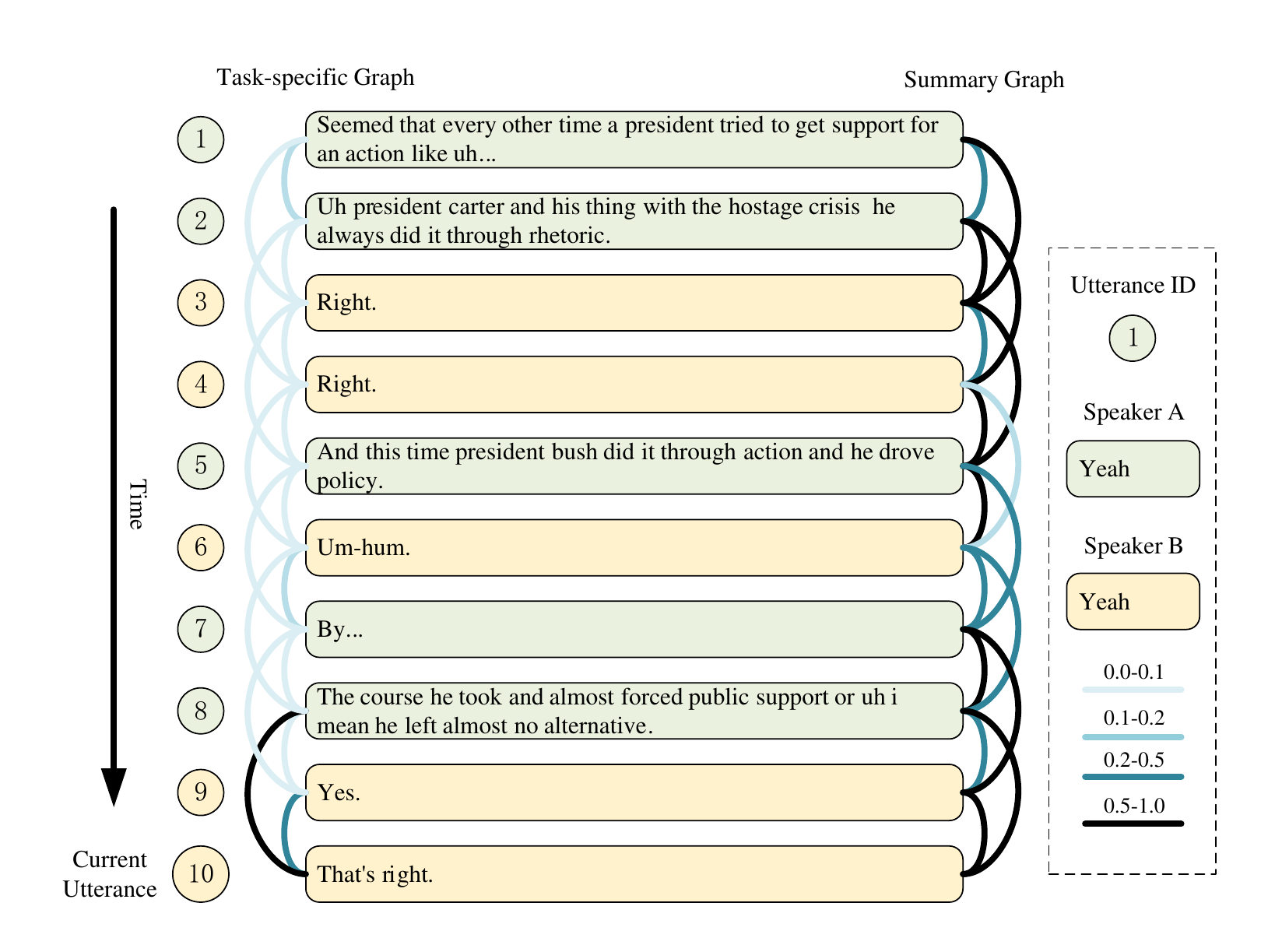}}
        \caption{Example of graphs generated by RTN: ten sequential utterances from "sw02130-A\_002749-003357" to "sw02130-B\_005687-005840" }
        \label{fig:case}
    \end{center}
\end{figure*}




\clearpage
\bibliography{abbrev,mybib}
\bibliographystyle{apalike}

